\pdfoutput=1
\documentclass[mnsc]{informs4}

\RequirePackage{tgtermes}
\RequirePackage{newtxtext}
\RequirePackage{newtxmath}
\RequirePackage{bm}
\RequirePackage{endnotes}

\OneAndAHalfSpacedXI

\usepackage{algorithm}
\usepackage{algpseudocode}
\usepackage{booktabs}
\usepackage{enumitem}

\usepackage{natbib}
 \bibpunct[, ]{(}{)}{,}{a}{}{,}%
 \def\bibfont{\small}%

\EquationsNumberedThrough
\TheoremsNumberedThrough
\ECRepeatTheorems

\newcommand{\R}{\mathbb{R}}

\newcommand{\E}{\mathbb{E}}
\newcommand{\Var}{\mathrm{Var}}
\newcommand{\Prob}{\mathbb{P}}
\newcommand{\KL}{\mathrm{KL}}
\newcommand{\Alt}{\mathrm{Alt}}
\newcommand{\cA}{\mathcal{A}}
\newcommand{\cD}{\mathcal{D}}

\MANUSCRIPTNO{}

\makeatletter
\def\theARTICLETOP{}%
\def\theLRHFirstLine{}%
\def\theLRHSecondLine{}%
\def\theRRHFirstLine{}%
\def\theRRHSecondLine{}%
\makeatother

\everydisplay{%
\abovedisplayskip=4pt plus 1.5pt minus 1.5pt
\abovedisplayshortskip=3pt plus 1.5pt minus 1.5pt
\belowdisplayskip=4pt plus 1.5pt minus 1.5pt
\belowdisplayshortskip=3pt plus 1.5pt minus 1.5pt
}

\begin{document}

\RUNAUTHOR{}
\RUNTITLE{}

\TITLE{Fundamental Limitations of Fixed-Budget Best-Arm Identification}

\ARTICLEAUTHORS{
\AUTHOR{Motti Goldberger}
\AFF{Yale University}
}

\ABSTRACT{
In fixed-budget best-arm identification, also known as ranking and selection, an algorithm has a sampling budget to distribute across $K$ arms. Each sample provides noisy feedback about that arm's mean, and the goal is to identify the arm with the largest mean. A common performance benchmark is the \emph{static oracle}: a non-adaptive strategy that knows the means in advance and chooses fixed sampling proportions to maximize the exponential decay rate of the probability of incorrect identification. Several adaptive algorithms have been constructed such that their sampling proportions converge to the static oracle proportions. However, it has remained open whether any algorithm could match the static oracle's error decay rate uniformly across all problem instances. We answer this in the negative. For any $K\ge 3$ and for rewards drawn from any one-parameter natural exponential family, we show that for any algorithm, there is at least one instance where the error decay rate is at most $\left(1 + \frac{\log(K)}{8}\right)^{-1}$ times that of the static oracle. This also answers the open question posed by \citet{qin2022open}, showing that fixed-budget best-arm identification does not admit a complexity.
}
\KEYWORDS{Ranking-and-Selection, Best-Arm Identification, Multi-Armed Bandits}

\maketitle

\section{Introduction}
Best-arm identification (BAI) in multi-armed bandits is the problem of a decision maker who sequentially collects samples to identify the best arm from a set of alternatives. Each arm generates independent rewards from a distribution with an unknown mean, and the goal is to identify the arm with the largest mean. Also referred to as Ranking-and-Selection, BAI arises in settings such as sequential selection in clinical trials, A/B testing, new product development, and simulation optimization. There are two dominant formulations of BAI. In \textit{fixed-confidence} identification, algorithms are designed to minimize the expected number of samples required to correctly identify the best arm at a prespecified confidence level. In \textit{fixed-budget} identification, the total number of samples is fixed in advance, and algorithms are designed to minimize the probability of incorrect identification. 

A BAI problem is specified by the unknown mean rewards of the alternatives; we call this specification a problem instance. A central goal is to characterize the difficulty of a given problem instance---how `hard' it is to identify the best arm. For \textit{fixed-confidence} identification, \citet{garivier2016optimal} provide a complete answer in the asymptotic regime. They show that there is an instance-dependent lower bound on the expected number of samples required to achieve a target confidence level, and that a single algorithm (Track-and-Stop) achieves this bound on every instance. When such a lower bound exists, and a single algorithm achieves it uniformly over all instances, we say the problem admits a complexity. 

The \textit{fixed-budget} setting is less understood. \citet{qin2022open} posed the open question of whether fixed-budget BAI has a complexity. \citet{degenne2023existence} formalized this question and showed that a complexity does not exist in certain special cases. In this paper, we show the fundamental result that when there are at least three arms with rewards drawn from any one-parameter natural exponential family (NEF), fixed-budget best-arm identification does not admit a complexity.

\subsection{The Fixed-Budget Setting}\label{subsec:fbs}
A problem instance with $K$ arms is described by its mean vector $\mu=(\mu_1,\ldots,\mu_K)$, where $i^*(\mu)$ denotes the (assumed unique) arm with the largest mean reward. An algorithm with fixed budget $T$ sequentially collects samples, then selects an arm
$\hat{i}_T$. On instance $\mu$,
\(
p_{\mu,T} := \Prob_\mu\big(\hat{i}_T \neq i^*(\mu)\big)
\)
denotes the probability of incorrect identification. The relevant asymptotic performance measure is \( \liminf_{T\to\infty}\, \frac{1}{T}\log\frac{1}{p_{\mu,T}}\), the exponential decay rate of the error probability. Informally, we say fixed-budget BAI has complexity $\Gamma^*(\mu)$ if the following two conditions are satisfied \citep{degenne2023existence}.
\begin{enumerate}
\item[(i)] \textbf{(Universal upper bound).} For every algorithm and every instance $\mu$,
\[
\liminf_{T\to\infty}\frac{1}{T}\log\frac{1}{p_{\mu,T}} \;\le\; \Gamma^*(\mu).
\]
\item[(ii)] \textbf{(Uniform achievability).} There exists a single algorithm such that for every instance $\mu$,
\[
\liminf_{T\to\infty}\frac{1}{T}\log\frac{1}{p_{\mu,T}} \;=\; \Gamma^*(\mu).
\]
\end{enumerate}

A natural benchmark to measure the performance of an algorithm is the \emph{static oracle} $\Gamma_{so}^*(\mu)$, which is, conceptually, the `best' error decay rate achievable by a non-adaptive strategy that samples each arm a fixed proportion of the time, with proportions chosen optimally given knowledge of $\mu$. \citet{glynn2004large} characterize the optimal static oracle proportions. However, the static oracle is not the answer to the complexity question: it is possible that
fully adaptive algorithms can achieve a better error decay rate. But to prove the non-existence of a complexity, the static oracle benchmark is sufficient by the following argument. Any candidate complexity $\Gamma^*(\mu)$ must satisfy
\(
\Gamma^*(\mu)\;\ge\; \Gamma_{so}^*(\mu)\) for all $\mu$, since condition (i) must hold for every algorithm, including
the static oracle strategy at $\mu$. Consequently, to rule out the existence of a complexity, it is enough to show that no single
algorithm can achieve $\Gamma_{so}^*(\mu)$ uniformly over all instances. In particular, if one proves that
for any algorithm there exists an instance $\mu$ such that
\[
\liminf_{T\to\infty}\frac{1}{T}\log\frac{1}{p_{\mu,T}}
\;\le\;
c\,\Gamma_{so}^*(\mu)
\quad\text{for some constant }c<1,
\]
then condition (ii) does not hold, and fixed-budget BAI does not have a complexity.

\subsection{Relation to Ranking and Selection}
The fixed-budget BAI problem arises in the ranking-and-selection/simulation-optimization literature,
where one allocates a finite simulation budget to maximize the probability of identifying the best arm (often referred to as the best system).
A common optimality target in adaptive algorithm design is for the empirical sampling proportions to converge to the optimal static allocation; see, e.g.,
\citep{shinbroadiezeevi2018,chenryzhov2019cei,chenryzhov2019bold,avci2021rateoptimal}. However, convergence of the allocation proportions to the static oracle proportions does not guarantee matching its error decay rate \citep{glynnjuneja2011,wuzhou2018}. Until this paper, it has remained an open question whether one can design a procedure whose error probability decays at the static oracle rate. We show that the answer is negative when $K \ge 3$: no adaptive procedure can match the static oracle rate uniformly over all instances.

\subsection{Prior Work}
Fixed-budget BAI has been studied extensively in the machine learning literature, with early work focusing on elimination algorithms \citep{audibert2010best,karnin2013almost}. Their guarantees are expressed in terms of gap-based hardness measures (e.g., $H_2(\mu)=\max_{j\in[K]} j/\Delta_{j}^2$ where $\Delta_k := \mu_1 - \mu_k$). In particular, the Successive Rejects algorithm achieves error bounds of the form $\exp\!\big(-\Theta(T/(H_2(\mu)\log K))\big)$. \citet{carpentier2016tight} derive lower bounds on the error exponent in terms of $H_2(\mu)$, showing that the Successive Rejects guarantees are on the order of minimax optimal. Under the minimax optimality criterion, \citet{komiyama2022minimax} characterize the optimal worst-case exponential decay rate for general hardness measure $H$. They also propose an algorithm designed to attain the optimal rate, but it is computationally prohibitive. More recently, \citet{wang2023ldp} develop tools for analyzing the performance of adaptive sampling algorithms from a large-deviation perspective, which sharpen the error-exponent analysis of Successive Rejects and motivate a new elimination algorithm (Continuous Rejects) with a better guarantee.

However, these results leave open the question of whether a single algorithm can match $\Gamma^*_{so}(\mu)$ uniformly over all instances \citep{qin2022open}. \citet{degenne2023existence} formalize this question through
difficulty ratios: the factor by which an algorithm's error decay rate falls short of the oracle on instance $\mu$. They show that there is no complexity when $K=2$ with Bernoulli rewards, as well as when $K \ge e^{80/3}$ and rewards follow a Gaussian distribution with variance $1$. \citet{kaufmann2016complexity} show that for $K = 2$ with Gaussian rewards, the optimal allocation is the Neyman allocation \citep{neyman1934two}, and the problem does have a complexity.

\subsection{Contribution}
For $K \geq 3$ with rewards from any one-parameter NEF, we show that for every adaptive algorithm, there exists an instance where the error decay rate is at most $\left(1 + \frac{1}{8}\log(K)\right)^{-1}$ times $\Gamma^*_{so}(\mu)$. It follows that for any class of algorithms that contains the static proportion algorithms, there is no complexity. This builds on \citet{degenne2023existence} to fill that gap when $3 \leq K \leq e^{80/3}$ for Gaussian rewards with variance $1$, and extends the negative result to all regular one-parameter NEFs. Table~\ref{tab:prior} summarizes what was known and the gap we fill.

\begin{table}[ht]
\TABLE
{Known results on the existence of a complexity.\label{tab:prior}}
{\begin{tabular}{@{}llll@{}}
\toprule
\up\textbf{Arms} & \textbf{Distribution} & \textbf{Complexity?} & \textbf{Reference} \\
\midrule
\up $K = 2$ & Gaussian & Yes & \citet{kaufmann2016complexity} \\
$K = 2$ & Bernoulli & No & \citet{degenne2023existence} \\
$K \geq e^{80/3}$ & Gaussian variance $1$ & No  & \citet{degenne2023existence} \\
\down $K \geq 3$ & One-parameter NEF & \textbf{No} & This paper \\
\bottomrule
\end{tabular}}
{}
\end{table}

\textbf{Organization.}
Section~\ref{sec:preliminaries} formally introduces fixed-budget best-arm identification, the static oracle,
and complexity.
Section~\ref{sec:gaussian} proves the negative result for Gaussian arms with variance $1$.  Section~\ref{sec:expfam} uses the argument from Section~\ref{sec:gaussian} together with a local quadratic approximation of the Kullback–Leibler (KL) divergence to extend the result to all one-parameter NEFs. Finally, Section~\ref{sec:disc} discusses implications of the result and surveys alternative theoretical targets suggested in recent work.

\section{Preliminaries}\label{sec:preliminaries}
Consider $K$ arms with unknown mean vector $\mu = (\mu_1, \ldots, \mu_K) \in \Theta^K$, where $\Theta \subseteq \R$ is an open interval. We assume the best arm is unique:
\[
i^*(\mu) := \argmax_{k \in [K]} \mu_k, \qquad [K] := \{1, \ldots, K\},
\]
and define $\Theta^K_{u} := \{\mu \in \Theta^K : i^*(\mu) \text{ is unique}\}$. Sampling arm $k$ yields an independent draw from distribution $\nu_{\mu_k}$, where $\{\nu_\theta : \theta \in \Theta\}$ is a known parametric family with mean $\theta$. An algorithm with fixed budget $T$ samples arms $A_1, \ldots, A_T \in [K]$ sequentially, possibly adaptively based on observed rewards, and outputs a selection $\hat{i}_T \in [K]$. We work with algorithm families $\cA=(\cA_T)_{T\ge1}$, where $\cA_T$ denotes the prescribed algorithm for budget $T$. The error probability of algorithm family $\cA$ on instance $\mu$ is
\(
p_{\mu,T}(\cA_T) := \Prob_\mu(\hat{i}_T \neq i^*(\mu)).
\)
Define the set of consistent algorithm families as
\[
\mathcal{C}_{\mathrm{cons}}(\Theta) := \Big\{\cA = (\cA_T)_{T \geq 1} : \forall\, \mu \in \Theta^K_{u},\ \lim_{T \to \infty} p_{\mu,T}(\cA_T) = 0\Big\}.
\]
Note that all static proportion algorithms that allocate nonzero sampling proportion to each arm are in $\mathcal{C}_{\mathrm{cons}}(\Theta)$ \citep{degenne2023existence}.

\subsection{The Static Oracle}
Consider fixed allocation proportions in the interior of the simplex $\omega = (\omega_1, \ldots, \omega_K) \in \Delta_K^0 := \{\omega \in (0,1)^K : \sum_k \omega_k = 1\}$. A static strategy samples arm $k$ approximately $\omega_k T$ times and recommends the arm with the highest empirical mean. Given a static allocation strategy $\omega$, the instance-dependent exponential error decay rate is given by
\[
\lim_{T \to \infty} \frac{1}{T} \log \frac{1}{p_{\mu,T}} = \Gamma(\mu, \omega), \quad \text{where} \qquad \Gamma(\mu, \omega) := \inf_{\lambda \in \Alt(\mu)} \sum_{k=1}^K \omega_k \, \KL(\lambda_k, \mu_k)
\]
and $\Alt(\mu) := \{\lambda \in \Theta^K : \exists\, j \neq i^*(\mu)\ \text{such that}\ \lambda_j \ge \lambda_{i^*(\mu)}\}$
is the set of alternative instances where $i^*(\mu)$ is not the unique best arm \citep{glynn2004large,degenne2023existence}. The static oracle chooses the best allocation $\omega$ knowing $\mu$, which yields the error decay rate and hardness
\[
\Gamma_{so}^*(\mu) := \sup_{\omega \in \Delta_K^0} \Gamma(\mu, \omega), \qquad H_{so}(\mu) := \frac{1}{\Gamma_{so}^*(\mu)}.
\]
Here $H_{so}(\mu)$ quantifies the asymptotic difficulty of the BAI task for the static oracle on instance $\mu$.
\subsection{Existence of Complexity}
Following \citet{degenne2023existence}, we formalize what it means for a complexity to exist. Define
\[
h_{\mu,T}(\cA_T) := \frac{T}{\log(1/p_{\mu,T}(\cA_T))} \in [0, \infty],
\]
with $h_{\mu,T} = 0$ if $p_{\mu,T} = 0$ and $h_{\mu,T} = +\infty$ if $p_{\mu,T} = 1$. For a benchmark difficulty function $H : \Theta^K_{u} \to (0, \infty)$, the \emph{difficulty ratio} $R_{H,T}(\cdot)$ measures how well algorithm family $\cA$ performs relative to $H$:
\[
R_{H,T}(\cA, \mu) := \frac{h_{\mu,T}(\cA_T)}{H(\mu)}, \qquad \text{and} \qquad R_{H,\infty}(\cA, \mu) := \limsup_{T \to \infty} R_{H,T}(\cA, \mu).
\]

\begin{definition}[Existence of Complexity]\label{def:complexity}
Let $\mathcal{C}$ be a class of algorithm families. A function $H:\Theta^K_{u} \to (0, \infty)$ is a \emph{complexity} for $\mathcal{C}$ iff:
\begin{enumerate}[label=(\roman*)]
\item $\inf_{\cA \in \mathcal{C}}\ \inf_{\mu \in \Theta^K_{u}} R_{H,\infty}(\cA, \mu) \geq 1$.
\item There exists $\cA^* \in \mathcal{C}$ such that $\sup_{\mu \in \Theta^K_{u}} R_{H,\infty}(\cA^*, \mu) \leq 1$.
\end{enumerate}
\end{definition}

In Definition~\ref{def:complexity}, condition (i) requires $H(\mu)$ to be a universal lower bound on the asymptotic hardness. Condition (ii) requires some algorithm to achieve this bound uniformly over all instances. The class of adaptive algorithms contains all static proportion algorithms, including the static oracle proportions at $\mu$, so if $H$ satisfies (i), then $\frac{H_{so}(\mu)}{H(\mu)} \geq 1$.  Therefore, if $H_{so}$ fails condition (ii), then $H$ must also fail this condition, and fixed-budget BAI does not admit a complexity. We establish the negative result by showing that for $K \geq 3$, $H_{so}$ fails condition (ii).
Our proofs rely on the following theorem from \citet{degenne2023existence}.

\begin{theorem}[{\citealp[Theorem 3]{degenne2023existence}}]\label{thm:comparison}
Fix a benchmark $H(\cdot) > 0$, an instance $\mu \in \Theta^K_{u}$, and a nonempty set $\cD(\mu) \subseteq \Alt(\mu) \cap \Theta^K_{u}$. For any $\cA \in \mathcal{C}_{\mathrm{cons}}(\Theta)$,
\begin{equation}\label{eq:thm1}
    \left( \sup_{\lambda \in \cD(\mu)} R_{H,\infty}(\cA, \lambda) \right)^{-1} \leq \sup_{\omega \in \Delta_K} \inf_{\lambda \in \cD(\mu)} \left\{ H(\lambda) \sum_{k=1}^K \omega_k \, \KL(\mu_k, \lambda_k) \right\}.
\end{equation}
\end{theorem}
Degenne's proof of Theorem~\ref{thm:comparison} uses the change of measure argument and data processing inequality common in the multi-armed bandit literature \citep{garivier2019explore,degenne2023existence} to bound the performance of any adaptive algorithm by a static optimization over $\omega\in\Delta_K$. To show that a complexity does not exist, we use this inequality with benchmark $H_{so}$ and construct $\mu$ and a set of instances $\cD(\mu) \subseteq \Alt(\mu) \cap \Theta^K_{u}$ such that the right-hand side is less than 1.

\section{Gaussian Rewards}\label{sec:gaussian}
As a building block for the more general result in Section~\ref{sec:expfam}, we start by considering the setting where each arm has Gaussian rewards with variance $1$, so \(\KL(\mathcal{N}(a, 1) \| \mathcal{N}(b, 1)) = \frac{(a-b)^2}{2}.\) Throughout this section, $\Theta = \R$. Lemma~\ref{lem:gaussian_oracle} states that for a static allocation $\omega \in \Delta_K^0$, the error decay rate is determined by the hardest-to-distinguish challenger arm $j \neq i^*(\mu)$, and that the optimal static oracle allocation balances sampling between the best arm and each challenger based on their distance.

\begin{lemma}[Gaussian Static Oracle]\label{lem:gaussian_oracle}
Let $\mu\in\Theta^K_{u}$ and $b=i^*(\mu)$. For any static algorithm with proportions $\omega \in \Delta_K^0$,
\[
\Gamma(\mu,\omega) = \min_{j\neq b}
\frac{\omega_b\omega_j}{\omega_b+\omega_j}\cdot \frac{(\mu_b-\mu_j)^2}{2} \qquad \text{and} \qquad 
\Gamma_{so}^*(\mu)
=\max_{\omega\in\Delta_K^0}\ \Gamma(\mu,\omega).
\]
\end{lemma}

\citet{degenne2023existence} proved that for Gaussian variance $1$ rewards, every consistent algorithm family $\cA$ satisfies
$\sup_{\mu\in\Theta^K_{u}} R_{H_{so},\infty}(\cA,\mu) \ge \tfrac{3}{80}\log K$. This exceeds $1$ when $K \ge e^{80/3}$, which implies no complexity in that regime.
Here we strengthen this by deriving a lower bound on the difficulty ratio that is of the same logarithmic order and is larger than $1$ for every $K\ge 3$.

\begin{theorem}\label{thm:gaussian}
For $K \geq 3$ and Gaussian variance $1$ rewards,
\[
\inf_{\cA \in \mathcal{C}_{\mathrm{cons}}(\Theta)}\ \sup_{\mu \in \Theta^K_{u}} R_{H_{so}, \infty}(\cA, \mu) \;\geq\;
1+\sum_{j=3}^K \frac{1}{(1+\sqrt{j-1})^2} > 1 + \frac{1}{8}\log(K).
\]
\end{theorem}
\proof{Proof:}
Fix arbitrary $\cA \in \mathcal{C}_{\mathrm{cons}}(\Theta)$.
We start by giving the principles of the proof. The key will be to construct a baseline instance $x$ and a family of instances $\{\lambda^A\}\cup\{\lambda^{B,3},\dots,\lambda^{B,K}\} \subseteq \Alt(x) \cap \Theta^K_{u}$.
Instance $\lambda^A$ differs from $x$ only on arms $1$ and $2$, so distinguishing $x$ from $\lambda^A$ requires substantial sampling of these two arms. In contrast, for each $j\in\{3,\dots,K\}$, instance $\lambda^{B,j}$ makes arm $j$ optimal and leaves all other arms unchanged. So distinguishing $x$ from $\lambda^{B,j}$ requires substantial sampling of arm $j$.  Since we do not know the true instance in advance,
we must allow for the possibility that $\mu$ is any of $\{\lambda^A\}\cup\{\lambda^{B,3},\dots,\lambda^{B,K}\}$.
For each such $\mu$, the baseline $x$ belongs to $\Alt(\mu)$, so correctly identifying the best arm on
$\mu$ requires collecting enough evidence to separate $\mu$ from $x$. But these separations rely on
different indices across the family (arms $1,2$ for $\lambda^A$ versus arm $j$ for $\lambda^{B,j}$),
so a single sampling allocation cannot be simultaneously optimal for all of them. We show this forces at
least one instance in the family to satisfy the first inequality in Theorem~\ref{thm:gaussian}.

The proof proceeds in the following steps. Step 1 constructs a sequence of instance families indexed by $L$ and gives intuition for why these instances are chosen. Step 2 shows that for each $L$, the right-hand side of~\eqref{eq:thm1} is bounded above by
\[
\left(1+\sum_{j=3}^K \frac{1}{(1+\sqrt{j-1})^2}\right)^{-1}+\,o_L(1).
\]
Step 3 applies Theorem~\ref{thm:comparison} and concludes by letting $L\to\infty$ so that the $o_L(1)$ terms vanish.

\noindent\textbf{Step 1: Instance construction.}
Fix an integer $L\ge K$ and define a baseline instance $x^{(L)}\in\Theta^K_{u}$ and alternative instances $\lambda^A,(\lambda^{B,j})_{j=3}^K\in\Theta^K_{u}$ by
\begin{align*}
x^{(L)}_k
&=
\begin{cases}
0, & k=1,\\
-1, & k=2,\\
-L^{k}, & 3\le k\le K,
\end{cases}
&
\lambda^{A}_k
&=
\begin{cases}
-\sqrt{L}, & k=1,\\
-1+\sqrt{L}, & k=2,\\
-L^{k}, & 3\le k\le K,
\end{cases}
&
\lambda^{B,j}_k
&=
\begin{cases}
0, & k=1,\\
-1, & k=2,\\
L^{j}\sqrt{L}, & 3\le k\le K,\ k=j,\\
-L^{k}, & 3\le k\le K,\ k\neq j.
\end{cases}
\end{align*}
Note $\cD_L := \{\lambda^A\}\cup\{\lambda^{B,3},\dots,\lambda^{B,K}\} \subseteq \Alt(x^{(L)}) \cap \Theta^K_{u}$.

\noindent 
Our goal is to show
\[
\sup_{\omega\in\Delta_K}\ \min_{\lambda\in\cD_L}\Big\{H_{so}(\lambda)\sum_{k=1}^K \omega_k\,\KL\!\big(x^{(L)}_k,\lambda_k\big)\Big\} < 1.
\]
The novel insight is to construct instances so that the minimum over $\lambda\in\cD_L$ reduces to a minimum of linear terms in $\omega$. In particular
\begin{equation}\label{eq:key_linear_min}
\min_{\lambda\in\cD_L}\Big\{H_{so}(\lambda)\sum_{k} \omega_k\,\KL(x^{(L)}_k,\lambda_k)\Big\}
\ \le\ (1+o_L(1))\min\Big\{\omega_1+\omega_2,\ c_3\omega_3,\dots,c_K\omega_K\Big\},
\end{equation}
for positive coefficients $\{c_j\}_{j=3}^K$.
Given that \eqref{eq:key_linear_min} holds, then taking the supremum w.r.t.\ $\omega$ and letting $L \to\infty$ shows the left-hand side of \eqref{eq:key_linear_min} is less than $\big(1+\sum_{j=3}^K c_j^{-1}\big)^{-1}$. This immediately gives us a difficulty ratio larger than one for all $K \geq 3$, which is our main objective. Within the set of instances that satisfy \eqref{eq:key_linear_min}, we get a tighter bound the larger $\sum_{j=3}^K c_j^{-1}$ is.

\smallskip
\noindent\textit{How we will get linear terms in \eqref{eq:key_linear_min}.}
For $\lambda^A$, only arms $1,2$ change by $\sqrt L$, so $\sum_k \omega_k\KL(x^{(L)}_k,\lambda^A_k)=(\omega_1+\omega_2)\cdot \frac{L}{2}$.
We will show $H_{so}(\lambda^A)$ is at most $2/L$ up to a $(1+o_L(1))$ factor, which makes the product $H_{so}(\lambda^A)\sum_k \omega_k\KL(x^{(L)}_k,\lambda^A_k)$ at most $\omega_1+\omega_2$ up to a $(1+o_L(1))$ factor.
For $\lambda^{B,j}$ only arm $j$ changes from $-L^j$ to $L^j\sqrt L$, thus
$\sum_k \omega_k\KL(x^{(L)}_k,\lambda^{B,j}_k)=\omega_j\cdot \frac{L^{2j+1}}{2}(1+o_L(1))$.
We will show $H_{so}(\lambda^{B,j})$ is at most $2(1+\sqrt{j-1})^2/L^{2j+1}$ up to a $(1+o_L(1))$ factor, so the $L^{2j+1}$ cancels and leaves a finite coefficient $c_j$ multiplying $\omega_j$.

\smallskip
\noindent\textit{Making $c_j$'s small.}
The coefficients $c_j$ are determined by the optimal static oracle allocation to differentiate $x^{(L)}$ and $\lambda^{B,j}$. On a high level, each $c_j$ is smaller when there are fewer contender arms close to the best, and the allocation can be concentrated on them. $\lambda^{B,j}$ and $x^{(L)}$ only differ at index $j$ and the values $x^{(L)}_k=-L^k$ are chosen so that $\lambda_j^{B,j} - \lambda_i^{B,j} = L^j\sqrt L \,(1+o_L(1))$, for arms $i = 1,\dots,j-1$, whereas every arm $k>j$ has difference $\lambda_j^{B,j} - \lambda_k^{B,j} = L^k\,(1+o_L(1))\gg L^j\sqrt L$.
Thus, it is only the $j-1$ `close' challengers that the oracle needs to spend non-negligible budget on, which gives the coefficients $c_j=(1+\sqrt{j-1})^2$.

\medskip
\noindent\textbf{Step 2: Verify Equation \eqref{eq:key_linear_min}.}
Fix $\omega \in \Delta_K$. For each $\lambda\in\cD_L$ we upper bound the quantity
$H_{so}(\lambda)\sum_{k=1}^K \omega_k\,\KL(x_k^{(L)},\lambda_k)$.
We begin by considering the static oracle hardness terms $H_{so}(\lambda^A)$ and $H_{so}(\lambda^{B,j})$.
To keep the main body of the paper focused, we state this as Lemma~\ref{lem:gaussian_promoted} and give its proof in the Appendix.
The main idea is to choose an explicit allocation proportion vector $\alpha$ to get an upper bound
$\Gamma(\lambda,\alpha)^{-1}$, and then use
$H_{so}(\lambda) = \Gamma_{so}^*(\lambda)^{-1} \leq \Gamma(\lambda,\alpha)^{-1}$.
\begin{lemma}\label{lem:gaussian_promoted}
For the instances $\lambda^A$ and $\lambda^{B,j}$ defined in Step~1:
\begin{enumerate}[label=(\alph*)]
\item $H_{so}(\lambda^A)\ \le\ \dfrac{2}{L}\big(1+o_L(1)\big)$.
\item For each $j\in\{3,\dots,K\}$,
\[
H_{so}(\lambda^{B,j})\ \le\ \dfrac{2(1+\sqrt{j-1})^2}{L^{2j+1}}\big(1+o_L(1)\big).
\]
\end{enumerate}
\end{lemma}

\noindent Next, we compute $\sum_{k=1}^K \omega_k\,\KL(x_k^{(L)},\lambda_k)$ for $\lambda=\lambda^A$ and for $\lambda=\lambda^{B,j}$. Since each alternative differs from $x^{(L)}$ at only a few indices, the sum collapses to a simple expression.
For alternative $\lambda^A$, only arms $1$ and $2$ change, each by $\sqrt{L}$, hence
\[
\sum_{k=1}^K \omega_k \, \KL(x_k^{(L)}, \lambda_k^A) = (\omega_1 + \omega_2)\cdot \frac{L}{2}.
\]
For alternative $\lambda^{B,j}$, only arm $j$ changes from $-L^j$ to $L^j\sqrt{L}$, hence
\[
\sum_{k=1}^K \omega_k \, \KL(x_k^{(L)}, \lambda_k^{B,j})
= \omega_j \cdot \frac{(L^j\sqrt{L}+L^j)^2}{2}
= \omega_j\cdot \frac{L^{2j+1}}{2}\big(1+o_L(1)\big).
\]
Combining these with Lemma~\ref{lem:gaussian_promoted} gives
\begin{align*}
H_{so}(\lambda^A)\sum_{k=1}^K \omega_k \, \KL(x_k^{(L)}, \lambda_k^A)
&\le (1+o_L(1))(\omega_1+\omega_2),\\
H_{so}(\lambda^{B,j})\sum_{k=1}^K \omega_k \, \KL(x_k^{(L)}, \lambda_k^{B,j})
&\le (1+\sqrt{j-1})^2(1+o_L(1))\,\omega_j,
\qquad j\in\{3,\dots,K\}.
\end{align*}
Defining $c_j:=(1+\sqrt{j-1})^2$ and taking the minimum over $\lambda\in\cD_L$ yields \eqref{eq:key_linear_min}.

\medskip
\noindent\textbf{Step 3: Apply Theorem~\ref{thm:comparison} and conclude.}
Applying Theorem~\ref{thm:comparison} with $\mu=x^{(L)}$, $\cD(\mu)=\cD_L$, and $H=H_{so}$, and then using \eqref{eq:key_linear_min}, gives
\begin{align*}
\left( \sup_{\lambda \in \cD_L} R_{H_{so}, \infty}(\cA, \lambda) \right)^{-1}
&\le
\sup_{\omega \in \Delta_K} \min_{\lambda \in \cD_L}
\left\{ H_{so}(\lambda) \sum_{k=1}^K \omega_k \, \KL(x_k^{(L)}, \lambda_k) \right\}\\
&\le (1+o_L(1))\sup_{\omega\in\Delta_K}\min\Big\{\omega_1+\omega_2,\ c_3\omega_3,\dots,c_K\omega_K\Big\}\\
&\le \left(1+\sum_{j=3}^K c_j^{-1}\right)^{-1} + o_L(1).
\end{align*}
The last inequality follows from Lemma~\ref{lem:simplex_general_cj}, which is proven in the Appendix. Since
$\sup_{\mu \in \Theta^K_{u}} R_{H_{so}, \infty}(\cA, \mu) \geq \sup_{\lambda \in \cD_L} R_{H_{so}, \infty}(\cA, \lambda)$ for every $L$,
letting $L \to \infty$ gives
\[
\sup_{\mu \in \Theta^K_{u}} R_{H_{so}, \infty}(\cA, \mu)
\ \ge\
1+\sum_{j=3}^K \frac{1}{(1+\sqrt{j-1})^2} > 1 + \frac{1}{8}\log(K), 
\]
of which the second inequality is proven in Lemma~\ref{lem:log_lower_bound} of the Appendix.
Since $\cA \in \mathcal{C}_{\mathrm{cons}}(\Theta)$ was arbitrary, taking $\inf_{\cA \in \mathcal{C}_{\mathrm{cons}}(\Theta)}$ gives the result.\Halmos
\endproof

\section{Extension to Natural Exponential Families}\label{sec:expfam}

We now generalize the result of Theorem~\ref{thm:gaussian} to settings where rewards are drawn from any regular one-parameter NEF.

\begin{definition}[Regular one-parameter natural exponential family]\label{def:expfam}
A family of distributions $\{\nu_\theta:\theta\in\Theta\}$ on $\R$ is a \emph{regular one-parameter NEF parameterized by its mean}
if there exist an open interval $\mathcal H\subset\R$, a probability distribution $\nu_0$ supported on $\mathcal X\subseteq\R$, and
$A\in C^2(\mathcal H)$ with $0<A''(\eta)<\infty$ for all $\eta\in\mathcal H$ such that, for every $\eta\in\mathcal H$,
\[
\frac{d\nu_\eta}{d\nu_0}(x)=\exp(\eta x-A(\eta)),\qquad x\in\mathcal X.
\]
Define $\theta(\eta):=\E_{\nu_\eta}[X]=A'(\eta)$, set $\Theta:=\theta(\mathcal H)$, and write $\nu_\theta:=\nu_{\eta(\theta)}$
where $\eta(\cdot)=\theta^{-1}(\cdot)$.
\end{definition}

Throughout this section, assume $\{\nu_\theta:\theta\in\Theta\}$ is a family of distributions as in Definition~\ref{def:expfam}, with $\Var_{\nu_\theta}(X)\in(0,\infty)$ for all $\theta\in\Theta$.
The key fact we use is that, on a sufficiently small neighborhood of any fixed $\theta_0\in\Theta$, the KL divergence is locally quadratic. Lemma~\ref{lem:expfam_local_KL} formalizes this local approximation, which allows us to use a similar argument to that in the proof of Theorem~\ref{thm:gaussian} if we can construct instances that are sufficiently close together.

\begin{lemma}\label{lem:expfam_local_KL}
Let $\{\nu_\theta:\theta\in\Theta\}$ be a regular one-parameter NEF parameterized by its mean, with $\Theta$ open.
Fix $\theta_0\in\Theta$ and let $v_0=\Var_{\nu_{\theta_0}}(X)\in(0,\infty)$.
Then for every sequence $r_L\downarrow 0$ there exists $\delta_L\to 0$ such that, for all sufficiently large $L$ and all
$\theta,\theta'\in(\theta_0-r_L,\theta_0+r_L)$,
\[
(1-\delta_L)\frac{(\theta-\theta')^2}{2v_0}
\ \le\
\KL(\theta,\theta')
\ \le\
(1+\delta_L)\frac{(\theta-\theta')^2}{2v_0}.
\]
\end{lemma}

\begin{theorem}\label{thm:expfam}
For $K \geq 3$ and reward distributions of a regular one-parameter NEF,
\[
\inf_{\cA \in \mathcal{C}_{\mathrm{cons}}(\Theta)}\ \sup_{\mu \in \Theta^K_{u}} R_{H_{so}, \infty}(\cA, \mu)
\ \ge\
1+\sum_{j=3}^K \frac{1}{(1+\sqrt{j-1})^2}.
\]
\end{theorem}
\proof{Proof:}
Fix arbitrary $\cA \in \mathcal{C}_{\mathrm{cons}}(\Theta)$.
The proof will use the arguments in the proof of Theorem~\ref{thm:gaussian}, but we first scale the instances so that all means are sufficiently close to a fixed $\theta_0\in\Theta$. The proof proceeds as follows. Step~1 chooses a sequence $\varepsilon_L\downarrow 0$ and defines a baseline $x^{(L)}$ together with a set of instances
$\cD_L\subseteq \Alt(x^{(L)}) \cap \Theta^K_{u}$ by taking the same instances as in the Gaussian proof (centered at $\theta_0$) and multiplying every deviation from $\theta_0$ by $\varepsilon_L$.
Step~2 then chooses a sequence $r_L\downarrow 0$ and verifies that, for $L$ large, every component in $x^{(L)}$ and in $\cD_L$ lies in $(\theta_0-r_L,\theta_0+r_L)$, which ensures Lemma~\ref{lem:expfam_local_KL} applies to every KL term that will appear.
Step~3 then uses Lemma~\ref{lem:expfam_local_KL} to replace the $\KL$ terms in the sums $\sum_k \omega_k\,\KL(x_k^{(L)},\lambda_k)$, as well as in $H_{so}(\lambda)$, with $(\theta-\theta')^2/(2v_0)$. Combining these yields an inequality that upper bounds~\eqref{eq:thm1}.
Step~4 then shifts the instances by $\theta_0$ and rescales by $\varepsilon_L$, so the instances align with those used in Theorem~\ref{thm:gaussian}, allowing the analysis of Theorem~\ref{thm:gaussian} to be reused.
Finally, Step~5 uses this together with Theorem~\ref{thm:comparison} and lets $L\to\infty$ to obtain the stated lower bound.

\noindent\textbf{Step 1: Instance construction.}
Fix an integer $L\ge K$ and set $\varepsilon_L:=L^{-(K+1)}$. Define the baseline $x^{(L)}\in\R^K$ and alternatives
$\lambda^A,(\lambda^{B,j})_{j=3}^K\in\R^K$ by
\begin{align*}
x^{(L)}_k
&=
\begin{cases}
\theta_0, & k=1,\\
\theta_0-\varepsilon_L, & k=2,\\
\theta_0-L^{k}\varepsilon_L, & 3\le k\le K,
\end{cases}
&
\lambda^{A}_k
&=
\begin{cases}
\theta_0-\sqrt{L}\,\varepsilon_L, & k=1,\\
\theta_0-\varepsilon_L+\sqrt{L}\,\varepsilon_L, & k=2,\\
\theta_0-L^{k}\varepsilon_L, & 3\le k\le K,
\end{cases}
&
\lambda^{B,j}_k
&=
\begin{cases}
\theta_0, & k=1,\\
\theta_0-\varepsilon_L, & k=2,\\
\theta_0+L^{j}\sqrt{L}\,\varepsilon_L, & k=j,\\
\theta_0-L^{k}\varepsilon_L, & k\neq j.
\end{cases}
\end{align*}
All of these instances lie in $\Theta$ for sufficiently large $L$, since the largest deviation from $\theta_0$ is at most $L^K\sqrt{L}\,\varepsilon_L=L^{-1/2}$. Thus for sufficiently large $L$, $\cD_L:=\{\lambda^A\}\cup\{\lambda^{B,3},\dots,\lambda^{B,K}\}\subseteq\Alt(x^{(L)})\cap\Theta^K_{u}$.

\noindent\textbf{Step 2: KL is ``Gaussian-like'' on our scale.}
Set $r_L:=L^{-1/3}$. By Lemma~\ref{lem:expfam_local_KL} there exists a sequence $\delta_L \to 0$ such that for all
$\theta,\theta'\in(\theta_0-r_L,\theta_0+r_L)$,
\[
(1-\delta_L)\,\KL_0(\theta,\theta')\ \le\ \KL(\theta,\theta')\ \le\ (1+\delta_L)\,\KL_0(\theta,\theta'),
\qquad
\text{where} \qquad \KL_0(\theta,\theta'):=\frac{(\theta-\theta')^2}{2v_0}.
\]
$r_L$ is chosen to decay slower than the instances do, so for sufficiently large $L$, all components of $x^{(L)}$ and all instances in $\cD_L$ lie in $(\theta_0-r_L,\theta_0+r_L)$.

\noindent\textbf{Step 3: Replace $\KL$ with $\KL_0$.}
At this point, Step~2 ensures that Lemma~\ref{lem:expfam_local_KL} applies for all instances we consider.
The goal of this step is to upper bound the quantity
\[
H_{so}(\lambda)\sum_{k=1}^K \omega_k\,\KL\!\big(x^{(L)}_k,\lambda_k\big)
\]
by replacing the $\KL$ in it by $\KL_0$.
There are two places where $\KL$ enters: the KL sum and the static-oracle hardness $H_{so}(\lambda)$.
We handle them in turn and then combine the bounds. Define
\[
\Gamma_0(\mu,\alpha):=\inf_{\lambda\in\Alt(\mu)}\sum_{k=1}^K \alpha_k\,\KL_0(\lambda_k,\mu_k),
\qquad
\Gamma_0^*(\mu):=\sup_{\alpha\in\Delta_K^0}\Gamma_0(\mu,\alpha),
\qquad
H_0(\mu):=\frac{1}{\Gamma_0^*(\mu)}.
\]

\smallskip
\noindent\textit{(i) KL sums.}
Fix $\lambda\in\cD_L$ and $\omega\in\Delta_K$. Applying Lemma~\ref{lem:expfam_local_KL} coordinate-wise gives
\[
\sum_{k=1}^K \omega_k\,\KL\!\big(x^{(L)}_k,\lambda_k\big)
\ \le\ (1+\delta_L)\sum_{k=1}^K \omega_k\,\KL_0\!\big(x^{(L)}_k,\lambda_k\big).
\]

\smallskip
\noindent\textit{(ii) Hardness terms.}
Fix $\lambda\in\cD_L$ and $\alpha\in\Delta_K^0$, and let $b=i^*(\lambda)$.
As in the proof of Lemma~\ref{lem:gaussian_oracle},
\[
\Gamma(\lambda,\alpha)
=
\min_{j\neq b}\ \inf_{m\in\Theta}\Big\{\alpha_b\,\KL(m,\lambda_b)+\alpha_j\,\KL(m,\lambda_j)\Big\},
\]
and the infimum is attained at some $m\in(\lambda_j,\lambda_b)$.
For $\lambda\in\cD_L$ and $L$ large, $(\lambda_j,\lambda_b)\subset(\theta_0-r_L,\theta_0+r_L)$, so Lemma~\ref{lem:expfam_local_KL} implies
\[
\alpha_b\,\KL(m,\lambda_b)+\alpha_j\,\KL(m,\lambda_j)
\ \ge\ (1-\delta_L)\Big\{\alpha_b\,\KL_0(m,\lambda_b)+\alpha_j\,\KL_0(m,\lambda_j)\Big\}.
\]
Taking $\inf_m$ and then $\min_{j\neq b}$ yields $\Gamma(\lambda,\alpha)\ge (1-\delta_L)\Gamma_0(\lambda,\alpha)$.
Taking $\sup_\alpha$ gives $\Gamma_{so}^*(\lambda)\ge (1-\delta_L)\Gamma_0^*(\lambda)$, i.e.
\[
H_{so}(\lambda)\ \le\ \frac{1}{1-\delta_L}\,H_0(\lambda).
\]
Combining (i) and (ii), for all $\lambda\in\cD_L$ and $\omega\in\Delta_K$,
\[
H_{so}(\lambda)\sum_{k=1}^K \omega_k\,\KL\!\big(x^{(L)}_k,\lambda_k\big)
\ \le\
\frac{1+\delta_L}{1-\delta_L}\,
H_0(\lambda)\sum_{k=1}^K \omega_k\,\KL_0\!\big(x^{(L)}_k,\lambda_k\big).
\]

\noindent\textbf{Step 4: Transform to instances in Theorem~\ref{thm:gaussian}.}
What remains is to control
\[
\sup_{\omega\in\Delta_K}\ \min_{\lambda\in\cD_L}
\left\{H_0(\lambda)\sum_{k=1}^K \omega_k\,\KL_0\!\big(x_k^{(L)},\lambda_k\big)\right\}.
\]
The point of this step is that, after shifting by $\theta_0$ and rescaling by $\varepsilon_L$, this quantity becomes
exactly the same Gaussian variance $1$ expression that was analyzed in the proof of Theorem~\ref{thm:gaussian}. Define the rescaled mean vectors
\[
\tilde x^{(L)}_k:=\frac{x^{(L)}_k-\theta_0}{\varepsilon_L},\qquad
\tilde\lambda^A_k:=\frac{\lambda^A_k-\theta_0}{\varepsilon_L},\qquad
\tilde\lambda^{B,j}_k:=\frac{\lambda^{B,j}_k-\theta_0}{\varepsilon_L}\ \ (j=3,\dots,K),
\]
and let $\tilde\cD_L:=\{\tilde\lambda^A\}\cup\{\tilde\lambda^{B,3},\dots,\tilde\lambda^{B,K}\}$.
By construction, $\tilde x^{(L)}$ and $\tilde\cD_L$ are the baseline and instance family used in the proof of Theorem~\ref{thm:gaussian}.

\smallskip
\noindent\textit{(i) The $\KL_0$ sums rescale to Gaussian variance $1$ KL.}
For any $\omega\in\Delta_K$ and any $\lambda\in\cD_L$,
\[
\sum_{k=1}^K \omega_k\,\KL_0\!\big(x_k^{(L)},\lambda_k\big)
=\sum_{k=1}^K \omega_k\,\frac{\big(\varepsilon_L(\tilde x_k^{(L)}-\tilde\lambda_k)\big)^2}{2v_0}
=\frac{\varepsilon_L^2}{v_0}\sum_{k=1}^K \omega_k\,\frac{(\tilde x_k^{(L)}-\tilde\lambda_k)^2}{2}.
\]

\smallskip
\noindent\textit{(ii) $H_0$ rescales by the same factor.}
Shifting and scaling preserve $\Alt(\cdot)$, so for any $\alpha\in\Delta_K^0$,
\[
\Gamma_0(\lambda,\alpha)
=\inf_{\nu\in\Alt(\lambda)}\sum_{k=1}^K \alpha_k\,\KL_0(\nu_k,\lambda_k)
=\frac{\varepsilon_L^2}{v_0}\,
\inf_{\tilde\nu\in\Alt(\tilde\lambda)}\sum_{k=1}^K \alpha_k\,\frac{(\tilde\nu_k-\tilde\lambda_k)^2}{2}.
\]
Therefore $\Gamma_0^*(\lambda)=\frac{\varepsilon_L^2}{v_0}\,\Gamma_{so}^{G,*}(\tilde\lambda)$, where
$\Gamma_{so}^{G,*}$ is the Gaussian variance $1$ static oracle decay rate, and hence
\[
H_0(\lambda)=\frac{1}{\Gamma_0^*(\lambda)}=\frac{v_0}{\varepsilon_L^2}\,H_{so}^G(\tilde\lambda),
\]
with $H_{so}^G$ corresponding to the Gaussian variance $1$ static oracle hardness function.

\noindent Combining (i) and (ii), for any $\omega\in\Delta_K$ and $\lambda\in\cD_L$,
\[
H_0(\lambda)\sum_{k=1}^K \omega_k\,\KL_0\!\big(x_k^{(L)},\lambda_k\big)
=
H_{so}^G(\tilde\lambda)\sum_{k=1}^K \omega_k\,\frac{(\tilde x_k^{(L)}-\tilde\lambda_k)^2}{2}.
\]
Thus the max--min expression over $\lambda\in\cD_L$ under $(H_0,\KL_0)$ is identical to the Gaussian variance $1$
max--min expression over $\tilde\lambda\in\tilde\cD_L$ with baseline $\tilde x^{(L)}$.
Applying the analysis from Theorem~\ref{thm:gaussian} we get
\[
\sup_{\omega\in\Delta_K}\ \min_{\lambda\in\cD_L}
\left\{H_0(\lambda)\sum_{k=1}^K \omega_k\,\KL_0\!\big(x_k^{(L)},\lambda_k\big)\right\}
\le
\left(1+\sum_{j=3}^K \frac{1}{(1+\sqrt{j-1})^2}\right)^{-1}
+o_L(1).
\]
\noindent\textbf{Step 5: Apply Theorem~\ref{thm:comparison} and conclude.}
We now have the elements needed to complete the proof. 
\begin{align*}
\left( \sup_{\lambda \in \cD_L} R_{H_{so}, \infty}(\cA, \lambda) \right)^{-1}
&\le \sup_{\omega \in \Delta_K}\ \min_{\lambda \in \cD_L}
\left\{ H_{so}(\lambda) \sum_k \omega_k \, \KL(x_k^{(L)}, \lambda_k) \right\}
&&\tag{Thm.~\ref{thm:comparison}}\\
&\le \frac{1+\delta_L}{1-\delta_L}
\sup_{\omega \in \Delta_K}\ \min_{\lambda \in \cD_L}
\left\{ H_0(\lambda)\sum_{k=1}^K \omega_k\,\KL_0\!\big(x^{(L)}_k,\lambda_k\big) \right\}
&&\tag{Step~3}\\
&\le \frac{1+\delta_L}{1-\delta_L}
\left(\left(1+\sum_{j=3}^K \frac{1}{(1+\sqrt{j-1})^2}\right)^{-1}+o_L(1)\right)
&&\tag{Step~4}\\
&=
\left(1+\sum_{j=3}^K \frac{1}{(1+\sqrt{j-1})^2}\right)^{-1}+o_L(1).
\end{align*}
Since $\sup_{\mu\in\Theta^K_{u}}R_{H_{so},\infty}(\cA,\mu)\ge \sup_{\lambda\in\cD_L}R_{H_{so},\infty}(\cA,\lambda)$ for every $L$,
letting $L\to\infty$ gives
\[
\sup_{\mu \in \Theta^K_{u}} R_{H_{so}, \infty}(\cA, \mu)
\ \ge\
1+\sum_{j=3}^K \frac{1}{(1+\sqrt{j-1})^2}.
\]
Since $\cA \in \mathcal{C}_{\mathrm{cons}}(\Theta)$ was arbitrary, taking $\inf_{\cA \in \mathcal{C}_{\mathrm{cons}}(\Theta)}$ gives the result.\Halmos
\endproof

\section{Discussion}\label{sec:disc}
This paper establishes the fundamental result that fixed-budget BAI does not admit a complexity. This answers the question posed by \citet{qin2022open} and \citet{degenne2023existence} of whether an analog to the fixed-confidence theory established by \citet{garivier2016optimal} can be extended to the fixed-budget setting. This result also has implications for the ranking-and-selection literature, where the static oracle is often used as a benchmark. We show that, for $K\ge 3$, the static oracle rate is unattainable as a uniform guarantee across all instances. In particular, for every consistent algorithm family $\cA$, there exists an instance $\mu$ such that 
\[
\liminf_{T\to\infty}\frac{1}{T}\log\frac{1}{p_{\mu,T}(\cA_T)}
\ \le\
\frac{1}{\,1+\sum_{j=3}^K \left(1+\sqrt{j-1}\right)^{-2}\,}\,\Gamma_{so}^*(\mu).
\]

Since matching the static oracle uniformly is not possible, it is natural to aim for other theoretical guarantees. Minimax analysis provides strong guarantees, but the established algorithms that achieve the optimal error exponent are computationally intractable \citep{komiyama2022minimax}. Another direction recently explored is large-deviation admissibility: compare algorithms by their instance-dependent error exponents and determine whether one procedure uniformly dominates another. For two-armed Bernoulli rewards, \citet{wang2024universally} show that uniform sampling is admissible even though a complexity does not exist. \citet{imbens2025admissibility} construct an adaptive elimination algorithm that dominates uniform sampling when $K\ge 3$ with Gaussian equal-variance rewards. Constructing relevant classes of algorithms, such as elimination or static oracle-tracking algorithms, and determining which algorithms are admissible within these classes, is a promising direction for future research.

\ACKNOWLEDGMENT{
    I would like to thank Nils Rudi for his mentorship and helpful discussions throughout this project.
}

\begin{APPENDICES}
\renewcommand{\thelemma}{\thesection.\arabic{lemma}}
\setcounter{lemma}{0}
\section{Proofs}\label{app:proofs}
\proof{Proof of Lemma~\ref{lem:gaussian_oracle}:}
Let $\mu\in\R^K$ have unique best arm $b=i^*(\mu)$. Fix $\omega\in\Delta_K^0$ and define
\[
F_\omega(\xi)
:=\sum_{k=1}^K \omega_k\,\KL\!\big(\mathcal N(\xi_k,1)\,\|\,\mathcal N(\mu_k,1)\big)
=\sum_{k=1}^K \omega_k\,\frac{(\xi_k-\mu_k)^2}{2}.
\]
By definition, $\Gamma(\mu,\omega)=\inf_{\xi\in\Alt(\mu)}F_\omega(\xi)$.
For each $j\neq b$ define $\Alt_j(\mu):=\{\xi\in\R^K:\xi_j\ge \xi_b\}$.
$\Alt(\mu)=\bigcup_{j\neq b}\Alt_j(\mu)$, so
\[
\inf_{\xi\in\Alt(\mu)}F_\omega(\xi)=\min_{j\neq b}\ \inf_{\xi\in\Alt_j(\mu)}F_\omega(\xi).
\]

Fix $j\neq b$. The constraint defining $\Alt_j(\mu)$ only involves $(\xi_b,\xi_j)$, and each term of $F_\omega$ is nonnegative and separable,
so for any feasible $(u,v)$ with $v\ge u$ we may set $\xi_b=u$, $\xi_j=v$, and $\xi_k=\mu_k$ for $k\notin\{b,j\}$. Thus
\[
\inf_{\xi\in\Alt_j(\mu)}F_\omega(\xi)
=
\inf_{(u,v)\in\R^2:\ v\ge u}
\left\{\omega_b\frac{(u-\mu_b)^2}{2}+\omega_j\frac{(v-\mu_j)^2}{2}\right\}.
\]
The objective is strictly convex in $(u,v)$, and its unconstrained minimizer is $(u,v)=(\mu_b,\mu_j)$, which violates $v\ge u$ because $\mu_b>\mu_j$.
Thus, the constrained minimizer lies on the boundary $u=v=:x$, and the problem reduces to
\begin{equation}\label{eq:infx}
\inf_{x\in\R}\left\{\omega_b\frac{(x-\mu_b)^2}{2}+\omega_j\frac{(x-\mu_j)^2}{2}\right\}.
\end{equation}
Differentiating w.r.t.\ $x$ and setting to zero gives
\[
(\omega_b+\omega_j)x=\omega_b\mu_b+\omega_j\mu_j
\quad\implies\quad
x^*=\frac{\omega_b\mu_b+\omega_j\mu_j}{\omega_b+\omega_j},
\]
Plugging this into~\eqref{eq:infx} gives
\[
\inf_{\xi\in\Alt_j(\mu)}F_\omega(\xi)
=
\frac{\omega_b\omega_j}{\omega_b+\omega_j}\cdot \frac{(\mu_b-\mu_j)^2}{2}.
\]
Therefore, for every $\omega\in\Delta_K^0$,
\[
\Gamma_{so}^*(\mu)
= \max_{\omega\in\Delta_K^0}\min_{j\neq b}\ \frac{\omega_b\omega_j}{\omega_b+\omega_j}\cdot \frac{(\mu_b-\mu_j)^2}{2}. \Halmos
\]
\endproof

\medskip
\proof{Proof of Lemma~\ref{lem:gaussian_promoted}:}
We upper bound $H_{so}(\lambda)=1/\Gamma_{so}^*(\lambda)$ by choosing an allocation $\alpha\in\Delta_K^0$
and lower bounding $\Gamma_{so}^*(\lambda)$ via Lemma~\ref{lem:gaussian_oracle}.
In particular, if $b=i^*(\lambda)$, then for any $\alpha\in\Delta_K^0$,
\begin{equation}\label{eq:lem1bound}
    \Gamma_{so}^*(\lambda)\ \ge\ 
\min_{i\neq b}\ \frac{\alpha_b\alpha_i}{\alpha_b+\alpha_i}\cdot \frac{(\lambda_b-\lambda_i)^2}{2}.
\end{equation}

\smallskip
\noindent\textbf{(a) The instance $\lambda^A$.}
Here $i^*(\lambda^A)=b=2$. Consider the allocation
\[
\alpha^A_1=\alpha^A_2=\frac{L-K+2}{2L},
\qquad
\alpha^A_k=\frac{1}{L}\ \ (k=3,\dots,K).
\]
We lower bound separately the $i=1$ term and the terms $i\in\{3,\dots,K\}$ in \eqref{eq:lem1bound}.
The overall minimum is at least the minimum of these lower bounds, so it suffices to show both are
$\ge \frac{L}{2}(1+o_L(1))$.

\smallskip
\noindent\textit{Arm $i=1$:} $\lambda_2^A-\lambda_1^A=2\sqrt{L}-1$, therefore
\[
\frac{\alpha^A_2\alpha^A_1}{\alpha^A_2+\alpha^A_1}\cdot \frac{(\lambda_2^A-\lambda_1^A)^2}{2}
=\frac{L-K+2}{4L}\Big(2L-2\sqrt{L}+\tfrac12\Big) 
=\frac{L-K+2}{2}\Big(1-\frac{1}{\sqrt{L}}+O(L^{-1})\Big)
= \frac{L}{2}\Big(1+o_L(1)\Big).
\]

\smallskip
\noindent\textit{Arms $i\in\{3,\dots,K\}$:}
Here $\lambda_2^A-\lambda_i^A=L^i+\sqrt{L}-1$, and 
\[
\frac{\alpha^A_2\alpha^A_i}{\alpha^A_2+\alpha^A_i}
=\frac{(L-K+2)/(2L)\cdot (1/L)}{(L-K+2)/(2L)+1/L}
=\frac{L-K+2}{L(L-K+4)}.
\]
Therefore,
\[
\frac{\alpha^A_2\alpha^A_i}{\alpha^A_2+\alpha^A_i}\cdot \frac{(\lambda_2^A-\lambda_i^A)^2}{2}
=\frac{L-K+2}{L(L-K+4)}\cdot \frac{(L^i+\sqrt{L}-1)^2}{2}
\ \ge\ \frac{L}{2}\big(1+o_L(1)\big).
\]
Hence $\Gamma_{so}^*(\lambda^A)\ge \frac{L}{2}(1+o_L(1))$, and
\[
H_{so}(\lambda^A)=\frac{1}{\Gamma_{so}^*(\lambda^A)}\ \le\ \frac{2}{L}\big(1+o_L(1)\big).
\]

\smallskip
\noindent\textbf{(b) The instance $\lambda^{B,j}$.}
Fix $j\in\{3,\dots,K\}$ and write $\lambda:=\lambda^{B,j}$, so $i^*(\lambda)=b=j$.
For $i<j$ we have $\lambda_i\in\{0,-1,-L^3,\dots,-L^{j-1}\}$, so
\[
\lambda_j-\lambda_i=L^j\sqrt{L}+O(L^{j-1})=L^j\sqrt{L}\big(1+o_L(1)\big),
\qquad \text{thus} \qquad 
(\lambda_j-\lambda_i)^2=L^{2j+1}\big(1+o_L(1)\big),
\]
while for $k>j$ we have $\lambda_j-\lambda_k=L^j\sqrt{L}+L^k\ge L^k$.

For large $L$ arms $1,\dots,j-1$ are the only `close' challengers (gap scales with $L^j\sqrt{L}$), so we choose $\alpha$ to balance arm $j$
against these $j-1$ arms. The arms $k>j$ have much larger gaps, so we assign them only a vanishing mass, but distribute it proportional to $(\lambda_j-\lambda_k)^{-2}$ so that the $k>j$ terms do not end up
controlling the minimum in \eqref{eq:lem1bound}. Let $\tau_L:=\mathbf{1}\{j<K\}/L$ and define $\alpha\in\Delta_K^0$ by
\[
\alpha_k=
\begin{cases}
\dfrac{1-\tau_L}{\sqrt{j-1}\,\bigl(1+\sqrt{j-1}\bigr)}, & 1\le k\le j-1,\\[6pt]
\dfrac{1-\tau_L}{1+\sqrt{j-1}}, & k=j,\\[8pt]
\tau_L\,\dfrac{(\lambda_j-\lambda_k)^{-2}}{\sum_{\ell=j+1}^{K}(\lambda_j-\lambda_\ell)^{-2}}, & j+1\le k\le K.
\end{cases}
\]
We lower bound the minimum in \eqref{eq:lem1bound} by treating $i<j$ and (when $j<K$) $k>j$ separately.

\smallskip
\noindent\textit{Arms $i<j$:}
For $i<j$, we have $\alpha_i=\alpha_j/\sqrt{j-1}$ and $(\lambda_j-\lambda_i)^2=L^{2j+1}(1+o_L(1))$ so
\[
\frac{\alpha_j\alpha_i}{\alpha_j+\alpha_i}\cdot \frac{(\lambda_j-\lambda_i)^2}{2}
=\frac{1-\tau_L}{2(1+\sqrt{j-1})^2}L^{2j+1}\big(1+o_L(1)\big) = \frac{L^{2j+1}}{2(1+\sqrt{j-1})^2}\big(1+o_L(1)\big).
\]

\smallskip
\noindent\textit{Arms $k>j$ (when $j<K$):}
Here $\alpha_k\le\tau_L=1/L$ while $\alpha_j\ge \frac{1}{2(1+\sqrt{j-1})}$ for all $L\ge 2$, so for $L$ large
$\alpha_k\le \alpha_j$ and thus $\alpha_j/(\alpha_j+\alpha_k)\ge 1/2$. So
\[
\frac{\alpha_j\alpha_k}{\alpha_j+\alpha_k}\cdot \frac{(\lambda_j-\lambda_k)^2}{2}\ \ge\ \frac{\alpha_k(\lambda_j-\lambda_k)^2}{4}.
\]
By the definition of $\alpha_k$,
\[
\alpha_k(\lambda_j-\lambda_k)^2=\frac{\tau_L}{S},
\qquad
S:=\sum_{\ell=j+1}^{K}(\lambda_j-\lambda_\ell)^{-2}.
\]
Using $\lambda_j-\lambda_\ell\ge L^\ell$ for $\ell>j$ gives
\[
S\le \sum_{\ell=j+1}^\infty L^{-2\ell}=\frac{L^{-2(j+1)}}{1-L^{-2}}\le \frac{4}{3}L^{-2j-2},
\]
thus
\[
\frac{\alpha_j\alpha_k}{\alpha_j+\alpha_k}\cdot \frac{(\lambda_j-\lambda_k)^2}{2}
\ \ge\ \frac{1}{4}\cdot \frac{1/L}{(4/3)L^{-2j-2}}
=\frac{3}{16}L^{2j+1}.
\]
Since $\frac{3}{16}>\frac{1}{2(1+\sqrt{j-1})^2}$ for every $j\ge 3$, the $k>j$ terms are strictly larger than the $i<j$
terms for $L$ large. Therefore the minimum in \eqref{eq:lem1bound} is attained among $i<j$, and we conclude
\[
\Gamma_{so}^*(\lambda)\ \ge\ \frac{L^{2j+1}}{2(1+\sqrt{j-1})^2}\big(1+o_L(1)\big),
\qquad
H_{so}(\lambda^{B,j})\ \le\ \frac{2(1+\sqrt{j-1})^2}{L^{2j+1}}\big(1+o_L(1)\big). \Halmos
\]
\endproof

\medskip
\proof{Proof of Lemma~\ref{lem:expfam_local_KL}:}
Because the family is regular, there exists an open interval $\mathcal H\subset\R$ and a $C^2$ strictly convex
log-partition function $A:\mathcal H\to\R$ such that for $\eta\in\mathcal H$,
\[
\frac{d\nu_{\eta}}{d\nu_0}(x)=\exp(\eta x-A(\eta)),
\qquad
\KL(\nu_{\eta}\|\nu_{\eta'})=A(\eta')-A(\eta)-A'(\eta)(\eta'-\eta).
\]
In this parameterization,
\[
\theta(\eta):=\E_{\nu_\eta}[X]=A'(\eta),
\qquad
\Var_{\nu_\eta}(X)=A''(\eta).
\]
Since $A''>0$, the map $\eta\mapsto\theta(\eta)$ is strictly increasing and thus invertible; write $\eta(\theta)$ for its inverse.
Let $\eta_0:=\eta(\theta_0)$, so $v_0=A''(\eta_0)$.

Fix $\theta,\theta'\in\Theta$ and let $\eta:=\eta(\theta)$ and $\eta':=\eta(\theta')$.
By Taylor's theorem applied to $A$ at $\eta$, there exists $c$ between $\eta$ and $\eta'$ such that
\begin{equation}\label{eq:kl_eta_quad}
\KL(\nu_{\eta}\|\nu_{\eta'})
=\frac12 A''(c)(\eta'-\eta)^2.
\end{equation}
By the mean value theorem applied to $A'$ there exists $d \in (\eta,\eta')$ such that
\begin{equation}\label{eq:mvt_theta}
\theta'-\theta=A'(\eta')-A'(\eta)=A''(d)(\eta'-\eta).
\end{equation}
Combining \eqref{eq:kl_eta_quad} and \eqref{eq:mvt_theta} gives
\begin{equation}\label{eq:kl_theta_ratio}
\KL(\theta,\theta')=\frac12\,\frac{A''(c)}{A''(d)^2}\,(\theta-\theta')^2.
\end{equation}
Let $r_L\downarrow 0$. By continuity of $\eta(\cdot)$ at $\theta_0$, \( s_L:=\sup\{|\eta(\theta)-\eta_0|:\ |\theta-\theta_0|\le r_L\}\ \to\ 0.\)
Define
\[
m_L:=\inf\{A''(u): |u-\eta_0|\le s_L\},\qquad
M_L:=\sup\{A''(u): |u-\eta_0|\le s_L\}.
\]
Then $m_L\to v_0$, $M_L\to v_0$, and $m_L>0$ for $L$ large.
For any $\theta,\theta'\in(\theta_0-r_L,\theta_0+r_L)$, we have $\eta,\eta'\in[\eta_0-s_L,\eta_0+s_L]$, and
$c,d\in[\eta_0-s_L,\eta_0+s_L]$,  therefore
\[
m_L\le A''(c)\le M_L,\qquad m_L\le A''(d)\le M_L.
\]
Thus
\[
\frac{m_L}{M_L^2}\ \le\ \frac{A''(c)}{A''(d)^2}\ \le\ \frac{M_L}{m_L^2}.
\]
Let
\[
\delta_L:=\max\left\{\left|\frac{v_0m_L}{M_L^2}-1\right|,\ \left|\frac{v_0M_L}{m_L^2}-1\right|\right\}.
\]
Then $\delta_L\downarrow 0$ and, for $L$ large,
\[
\frac{1-\delta_L}{v_0}\ \le\ \frac{A''(c)}{A''(d)^2}\ \le\ \frac{1+\delta_L}{v_0}.
\]
Plugging this into \eqref{eq:kl_theta_ratio} gives, for all $\theta,\theta'\in(\theta_0-r_L,\theta_0+r_L)$,
\[
(1-\delta_L)\frac{(\theta-\theta')^2}{2v_0}
\ \le\
\KL(\theta,\theta')
\ \le\
(1+\delta_L)\frac{(\theta-\theta')^2}{2v_0}. \Halmos
\]
\endproof

\subsection{Auxiliary Lemmas}\label{app:auxiliary}

\begin{lemma}\label{lem:simplex_general_cj}
Let $K\ge 3$ and let $c_3,\dots,c_K>0$. Then
\[
\sup_{\omega\in\Delta_K}\min\Big\{\omega_1+\omega_2,\ c_3\omega_3,\dots,c_K\omega_K\Big\}
=
\left(1+\sum_{j=3}^K\frac{1}{c_j}\right)^{-1}.
\]
\end{lemma}

\proof{Proof:}
Let $t:=\min\{\omega_1+\omega_2,\ c_3\omega_3,\dots,c_K\omega_K\}$.
Then $\omega_1+\omega_2\ge t$ and $\omega_j\ge t/c_j$ for all $j\ge 3$, so
\[
1=\sum_{k=1}^K\omega_k
\ge
t+\sum_{j=3}^K\frac{t}{c_j}
=
t\left(1+\sum_{j=3}^K\frac{1}{c_j}\right).
\]
Thus $t\le \left(1+\sum_{j=3}^K\frac{1}{c_j}\right)^{-1}$.
Equality is reached at  $t^* := \left(1+\sum_{j=3}^K\frac{1}{c_j}\right)^{-1}$,
choosing $\omega_j=t^*/c_j$ for $j\ge 3$, and $\omega_1+\omega_2=t^*$.\Halmos
\endproof

\begin{lemma}\label{lem:log_lower_bound}
For every integer $K\ge 3$,
\[
\sum_{j=3}^K \frac{1}{(1+\sqrt{j-1})^2}\ >\ \frac18\log K.
\]
\end{lemma}

\proof{Proof:}
For $K=3$, \(\frac{1}{6} > \frac{\log(3)}{8} \).
Now assume $K\ge 4$.
\[\sum_{j=3}^K \frac{1}{(1+\sqrt{j-1})^2}
\ > \ \sum_{j=3}^K \frac{1}{4(j-1)}
\ >\ \frac14\int_2^K \frac{1}{x}\ dx
\ = \ \frac14\log\!\Big(\frac{K}{2}\Big) \ \ge \ \frac18\log K. \Halmos
\]
\endproof

\end{APPENDICES}

\bibliographystyle{informs2014}
\bibliography{references}

\end{document}